\documentclass[letterpaper]{article} 
\usepackage[preprint]{aaai2027}  
\usepackage[hyphens]{url}  
\usepackage{graphicx} 
\usepackage{natbib}  
\usepackage{caption} 
\usepackage{booktabs}

\usepackage{amsmath, amssymb, amsthm}

\newtheorem{definition}{Definition}
\newtheorem{proposition}{Proposition}

\title{Learning Principal-Agent Contracts for Equitable Smallholder Carbon Farming under Moral Hazard and Adverse Selection}
\author{
    Rishi Bharadwaj,
    Yadati Narahari
}
\affiliations{
    Department of Computer Science and Automation, Indian Institute of Science (IISc) 
}

\begin{document}

\maketitle

\begin{abstract}
Agricultural soils are a major untapped carbon sink. Carbon farming is emerging as a promising practice for tapping this potential. Smallholder farmers, who dominate agriculture across South Asia and sub-Saharan Africa, are key to scaling climate mitigation via carbon farming. It is ironic that real-world carbon programs largely fail to reach them. We study this important gap through the lens of contract design. An aggregator offers a single pooled contract to a heterogeneous population of smallholder farmers who have private adoption costs (adverse selection) and exert unobserved effort (moral hazard), with agronomic outcomes evolving over multiple seasons. We formulate this evolving contracting problem as a POMDP and use reinforcement learning to learn a dynamic profit-maximising contract. We analyse the performance of the aggregator under various conditions. We find that a profit-maximising aggregator does not merely inherit the exclusion of smallholders, it amplifies it. On large farms the aggregator realises
87.7\% of achievable adoption, against only 8.2\% on smallholdings.
Per-hectare Measurement, Reporting and Verification (MRV) costs fall as farm size rises, and the aggregator's
pooling contract compounds this gradient rather than offsetting it.
A counterfactual that makes MRV costs purely area-proportional eliminates this disparity. Our results and simulation can guide contract and policy design that opens carbon income to smallholders while enabling agricultural soils to contribute to climate mitigation at scale.

\end{abstract}

\begin{links}
    \link{Code}{[https://github.com/Rishi-Bharadwaj/carbon-farming-contract-design]}
\end{links}

\section{Introduction}

Carbon farming offers a rare opportunity to address two global challenges simultaneously, removing atmospheric carbon while directing climate finance to vulnerable farming communities. The technical potential of soil carbon sequestration in agroecosystems is $4.4$--$11.4$ Gt CO$_2$e/yr~\citep{LAL2011}, but realised sequestration under existing carbon programs falls far short, in part because those programs do not reach smallholders. Smallholders under 2 hectares account for approximately 500 million of the world's 582 million farms~\citep{lowder_global_2025} and dominate agricultural landscapes across South Asia and sub-Saharan Africa. They sustain crop diversity that large monocultural operations erode \citep{Ricciardi}. Smallholders are therefore where the largest share of unrealised
sequestration potential and the greatest climate vulnerability
coincide. A market design that reaches them can bring a structurally excluded population into carbon markets, unlocking soil carbon mitigation at scale. Yet existing carbon-farming programs are systematically biased toward larger landholders~\citep{Barbato2023FarmerPerspectives,Cariappa2025,Johansson2025}.

Designing contracts that reach smallholders is difficult for several reasons. Adoption costs are privately known (adverse selection), effort is only partially observable (moral hazard), agronomic outcomes evolve over multiple seasons, and practical constraints require a single pooled contract to serve heterogeneous farmers. Fixed Measurement, Reporting and Verification (MRV) costs further disadvantage small farms by increasing per-hectare costs. The combination of private information, hidden effort, heterogeneous farmers, and multi-season dynamics creates a contract-design problem that is difficult to solve analytically and calls for methods that can learn adaptive policies under partial observability. Yet AI research has largely focused on firm-level carbon markets, leaving AI-driven contract design for carbon farming largely unexplored~\citep{welsh2026,wang_2024,priyanka2025carbonfarming}.

We model carbon farming as a principal-agent Stackelberg game in which an aggregator commits to a single pooled contract and farmers best respond myopically. We derive the equilibrium for the single-season game analytically. The joint presence of moral hazard and adverse selection makes optimal contract computation APX-hard in discrete settings~\citep{guruganesh_2021}. While our formulation is continuous, it further adds multi-season dynamics over a heterogeneous partially observable population. We therefore formulate a POMDP for which reinforcement learning is a standard approach. We model commercially rational aggregators because this reflects prevailing market practice, and examine how these incentives shape market participation.

Beyond learning contracts, our framework provides insight into who participates in carbon markets and why. Fixed MRV costs create a mechanical participation threshold that disadvantages smaller farms. We ask whether a profit-maximising aggregator merely inherits this threshold or amplifies it. To answer this, we benchmark against a first-best baseline computed under the same MRV cost structure and define the Realised Adoption Share (RAS) as the proportion of first-best adoption achieved by the aggregator. RAS reaches 87.7\% for large farms but only 8.2\% for smallholdings, showing exclusion well beyond what cost alone explains. We attribute this amplification to the interaction between profit maximisation and the pooling constraint. A counterfactual that equalises MRV distribution removes this disparity while increasing both farmer participation and enrolled area. Finally, we show that moral hazard can act as a cost-limiting mechanism rather than a pure efficiency loss, so introducing it alongside adverse selection need not reduce contract performance. Our framework also serves as a decision-support tool, letting stakeholders evaluate policy interventions before implementation. It can be recalibrated to other regions and market conditions.

\paragraph{Contributions.}
Learning approaches to contract design treat moral hazard and adverse
selection in isolation and assume stationary settings. Analytical
work on agricultural contracts handles both frictions but under stationary settings with a single
dimension of heterogeneity. We aim to fill this gap and aid smallholder participation. Our paper develops a conceptual framework for examining how commercially rational contract design shapes participation and distributional outcomes.
\begin{itemize}
\item We formulate commercial carbon-farming contract design as a
multi-season principal-agent Stackelberg game embedded in a POMDP. We
learn a dynamic pooled contract under joint moral hazard and adverse
selection across a population heterogeneous in adoption cost, farm
size and initial soil carbon.

\item We train a deep reinforcement-learning aggregator in this
environment and measure it against a first-best welfare benchmark.
This lets us ask not only which contracts the
aggregator learns, but which farmers those contracts end
up serving and why.

\item We introduce Realised Adoption Share (RAS), which separates
participation losses caused by the MRV cost structure from those
caused by profit maximisation under pooling.  Using RAS, we show that commercially rational aggregators amplify rather than merely inherit smallholder exclusion. We identify socialisation of MRV costs as a policy intervention that
equalises participation.
\item We release a configurable, open-source simulator that separates
agronomic, population and market parameters from the environment
implementation, so recalibration to a new region requires
configuration changes alone. It provides a reusable testbed for
evaluating carbon policy interventions before field deployment.
\end{itemize}

\section{Related Work}

\paragraph{Automated contract design.}Learning principal-agent contracts can be seen as a subfield of automated mechanism design~\citep{conitzer}, an area that has seen substantial recent progress~\citep{dutting_2019,scheid2025,wu2024,gerstgrasser2023}. For a comprehensive survey, see \citet{dutting_2024}.

Computing optimal contracts under joint frictions is difficult in general. \citet{guruganesh_2021} prove that the principal-agent problem with moral hazard and adverse selection is APX-hard in discrete combinatorial settings, for both the profit-maximizing single contract and the profit-maximizing menu of contracts.

Existing learning approaches each address only part of the problem. \citet{ivanov2024,wu2024} propose RL solutions for learning contracts, but neither considers adverse selection, while \citet{gerstgrasser2023} explore learning Stackelberg equilibria in multi-agent settings, without either friction. Conversely, \citet{bollini2026learning} consider learning in Bayesian Stackelberg games with adverse selection but without moral hazard, and \citet{scheid2025} propose algorithmic solutions for a bandit contracting setting that likewise includes only adverse selection.

\paragraph{Carbon contract design.}Carbon sequestration has likewise motivated work on contract design. \citet{ian_2012} study contracts with moral hazard and apply them to the problem of ensuring permanence in carbon sequestering activities. \citet{priyanka2025carbonfarming} provide a comprehensive survey of carbon farming that explicitly identifies contract design with AI as a promising future direction. \citet{Raina2024} examine how different incentive mechanisms for carbon farming are designed and implemented. \citet{wang_2024} use hierarchical MARL to model a government allocating credits to self-interested trading enterprises in a cap-and-trade system, while \citet{welsh2026} characterize Nash equilibria for greenhouse gas offset credit markets using Nash-DQN.

\paragraph{Most closely related work.} \citet{dai_li} and \citet{rob_2005} both study agricultural principal-agent contracts under moral hazard and adverse selection. However, they model simpler, stationary formulations with one level of heterogeneity, allowing analytical treatments. Our setting introduces heterogeneity across multiple farmer attributes, namely adoption costs, farm sizes and initial Soil Organic Carbon (SOC) concentration. More fundamentally, it adds non-stationary multi-season dynamics, which motivate our learning approach.
\section{Model}

We begin with a single aggregator and a single farmer. The farmer is risk-neutral and protected by limited liability so that all payments are non-negative. Key parameters are summarised in Table~\ref{tab:notation}.

\begin{table}[t]
\centering
\begin{tabular}{ll}
    $B$              & Set of possible carbon sequestering practices \\
    $P$              & Carbon price (\$/$tCO_2e$) \\
    $f$              & Farm size (hectares) \\
    $c_i$            & Carbon sequestered per practice $i$ ($tCO_2e$/ha)\\
    $k_i$            & Adoption cost for practice $i$ (\$/ha) \\
    $y_i$            & Profit due to yield change per practice $i$ (\$/ha) \\
    $m$              & Total MRV cost per farm (\$) \\
    $u_0$            & Farmer's outside option (reservation utility) \\
\end{tabular}
\caption{Key notation.}
\label{tab:notation}
\end{table}

MRV cost per farm is comprised of both a fixed component ($m_0$) and a variable area-dependent component ($m_1$) \citep{ducos2009}, with the $\delta$ term capturing the effect of scaling efficiency~\citep{scale}.
\begin{equation}
m = m_0 + m_1 \cdot f^{\delta}
\label{eq:mrv_cost}
\end{equation} 
We instantiate the environment with agronomic parameters for Indian paddy smallholders, following the agricultural literature (Supplementary Document (SD) E).

\begin{definition}[Action-Based Contract]
An action-based contract is a tuple $(\mathbf{a}, \alpha)$ where $\mathbf{a} = (a_1, \ldots, a_{|B|}) \in [0,a_{max}]^{|B|}$ is the vector of per-practice action payments and $a_{max}$ is the maximum per-action payment allowed. $\alpha \in [0,1]$ is the farmer's share of MRV costs.
\end{definition}

\subsection{Single-Season Game Structure}

The game proceeds as follows: (1) the aggregator offers a contract $(\mathbf{a}, \alpha)$; (2) the farmer observes the contract and chooses a subset of practices $S \subseteq B$ to adopt as well as whether or not to accept the contract; (3) payoffs are realised.

\paragraph{Farmer utility.} Given contract $(\mathbf{a}, \alpha)$, the farmer's utility from adopting practice set $S$ is:
\begin{equation}
    U_f(S) = \sum_{i \in S} \left( y_i - k_i + a_i \right) f - \alpha \cdot m
\end{equation}

Practice $i$ is adopted iff it contributes non-negatively:
\begin{equation}
    S = \left\{ i \in B : y_i - k_i + a_i \ge 0 \right\}
\end{equation}
The farmer accepts the contract iff $U_f(S) \geq u_0$, breaking ties in favor of the principal per standard convention. Here $S_0 = \{i \in B : y_i - k_i \geq 0\}$ is the farmer's outside option practice set and $u_0 = f\sum_{i \in S_0}(y_i - k_i)$. Farmers rejecting the contract may still adopt practices in $S_0$ independently.

\paragraph{Aggregator utility.}
Conditional on contract acceptance, the aggregator's utility is
\begin{equation}
    U_A(S) = P \cdot f \sum_{i \in S} c_i
    - f\sum_{i \in S} a_i - (1 - \alpha) \cdot m
\end{equation}
Otherwise, the aggregator's utility is $0$. Henceforth, all aggregator objectives are implicitly restricted to accepting farmers.

\subsection{Equilibrium Analysis}

\begin{proposition}
At any solution to the aggregator's problem in which the adoption
constraint is slack, the Individual Rationality constraint binds, thus
equilibrium farmer utility satisfies $U_f^* = u_0$. 

(Standard result, proof available in SD-A.)
\end{proposition}

Substituting the binding IR constraint into the aggregator's objective yields the equilibrium aggregator utility:
\begin{equation}
    U_A^* = P \cdot f \sum_{i \in S} c_i
           + f \sum_{i \in S} (y_i - k_i)
           - f \sum_{i \in S_0} (y_i - k_i) - m
    \label{eqn:aggregator-surplus}
\end{equation}

\section{Extensions}

\subsection{Moral Hazard}

The aggregator observes adopted practices but not the effort devoted to each. The farmer privately chooses $e_i \in [\underline{e}, 1]$ for each adopted practice $i$. We set $\underline{e} = 0.2$, reflecting that some effort is unavoidable under MRV. We perform a sensitivity analysis on $\underline{e}$ in SD-G, to show our conclusions are not an artifact of this choice.
Effort scales cost, yield and carbon outcomes: realised adoption cost is $e_i \cdot k_i$, realised yield $e_i \cdot y_i$,  and realised sequestration $e_i \cdot c_i$.
Under a pure action-based contract ($R = 0$), the farmer has no incentive to exert effort beyond the minimum, so $e_i = \underline{e}$ for all practices that are not individually profitable. To restore effort incentives, the aggregator augments the action-based contract with a result based payment $R$, yielding the hybrid contract $(\mathbf{a}, \alpha, R)$. The farmer's utility is:
\begin{equation}
    \label{eqn:mh_farmers_utility}
    \begin{split}
    U_f(S, \mathbf{e}) ={} &
    \sum_{i \in S} \left( e_i \cdot y_i - e_i \cdot k_i + a_i \right) f
    - \alpha \cdot m \\
    & + \left( \sum_{i \in S} e_i \cdot c_i \right) f R
    \end{split}
\end{equation}

Since the farmer's per-practice objective is linear in $e_i$, the best response is a threshold policy:
\begin{equation}
    \label{eqn:effort_best_response}
    e_i =
    \begin{cases}
        1 & \text{if } y_i + c_i R \geq k_i \\
        \underline{e} & \text{otherwise}
    \end{cases}
\end{equation}

The aggregator's utility under the hybrid contract is:
\begin{equation}
    \label{eqn:mh_agg_utility}
    U_A(S,\mathbf{e}) =
    f (P - R)\sum_{i \in S} e_i \cdot c_i
    - f \sum_{i \in S} a_i
    - (1-\alpha) \cdot m
\end{equation}

\subsection{Adverse Selection}

Adoption costs vary across farmers due to differences in soil quality, topography, access to capital, and local labour conditions \citep{eagle2022realizable,antle2003}. We model contracting  using a single pooled contract offered to the entire farmer population, consistent with current commercial practice. Prior work~\citep{dai_li, rob_2005} similarly studies pooling contracts rather than type-separating menus. Following \citet{rob_2005}, we model the effective adoption cost for farmer $j$ and practice $i$ as:
\begin{equation}
    k_{ji}^{\text{eff}} = \theta_j \cdot k_i, \qquad
    \theta_j \sim \text{Uniform}[0, 1]
\end{equation}

The scaling factor $\theta_j$ is private information, thus the aggregator cannot condition payments on the farmer's realised type.

\subsection{Multi-Season Dynamics}

A distinguishing feature of our framework is the explicit modelling of how current contracting decisions shape future states, captured by the following.

\paragraph{Yield ramp-up.} Following the modelling of \citet{dabbert1986} and the results of \citet{datta2025}, we model yield as ramping up linearly to its equilibrium from season $t=0$ to season $T_{L}=3$.
\begin{equation}
    \label{eqn:yield_ramp_up}
    y_{ji}^{t_{ji}} = \begin{cases}
        y_{i} & \text{if } y_{i} \leq 0 \\[4pt]
        y_{i} \cdot \min\left(\dfrac{t_{ji}}{T_{L}},\, 1\right)
                  & \text{if } y_{i} > 0
    \end{cases}
\end{equation}
where $t_{ji}$ is the number of seasons farmer $j$ has performed practice $i$. Negative yield effects are realised immediately, while positive effects emerge over time.

\paragraph{Adoption cost decay.} Motivated by evidence that imperfect knowledge of
new-technology management constitutes a significant adoption barrier that diminishes
rapidly with accumulated experience \citep{foster1995}, we model adoption costs as
decaying linearly toward their type-scaled equilibrium values $k^{\mathrm{eff}}_{ji}$
over $T_{L} = 3$ seasons:
\begin{equation}
    \label{eqn:adoption_cost_ramp_up}
    k_{ji}^{t_{ji}} = k^{\mathrm{eff}}_{ji} + \tfrac{1}{2}\,|k^{\mathrm{eff}}_{ji}| \cdot
    \max\left(1 - \frac{t_{ji}}{T_{L}},\, 0\right)
\end{equation}

\paragraph{Soil organic carbon dynamics.} We adopt the C-saturation model of \citet{stewart2007} with decomposition rate $h = 0.02$\,yr$^{-1}$ from \citet{coleman1996}:
\begin{equation}
    \label{eqn:soc_dynamic}
    C_j^{t+1} = \max\left(0,\;
        C_j^t
        + \tilde{I}_j^t \cdot \left(1 - \frac{C_j^t}{C_m}\right)
        - h \cdot C_j^t
    \right)
\end{equation}
where $C_m$ is the soil-type-specific saturation capacity and the saturation deficit captures diminishing returns as soils approach capacity. Realized carbon input is the sum of effort-dependent practice contributions and a stochastic shock. As effort is unobservable and the shock breaks the deterministic link between effort and outcome, effort is not directly contractible, creating moral hazard. Farmers make decisions based on the expected carbon sequestered.
\begin{equation}
    \label{eqn:soc_input_noise}
    \tilde{I}_j^t = \sum_{i \in S_j^{t}} e_{ji}^t \cdot  c_{i} + \varepsilon_j^t,
    \qquad \varepsilon_j^t \sim \mathcal{N}(0,\, \sigma^2)
\end{equation}

\subsection{POMDP Formulation}

We extend the single-season Stackelberg game to a multi-season POMDP in which the
aggregator is an RL agent offering a single pooled contract to $N$ heterogeneous
farmers over a horizon of $T$ seasons. Farmers best-respond myopically each season.

\paragraph{Observation space.} Under pooled contracting, the aggregator observes the farmer population through aggregate statistics of the population. These aggregate observations condition the contract offered in the following season and define the observation at state $t$:
\begin{equation}
    o^t = \left(
        \tau^t,\;
        \bar{f},\;
        \overline{\mathrm{SOC}}^t,\;
        \bar{\tilde{c}}^t,\;
        \bar{d}^t,\;
        \bar{\rho}^t
    \right)
\end{equation}
where $\tau^t$ is the season index, $\bar{f}$ the mean farm size,
$\overline{\mathrm{SOC}}^t$ the mean soil organic carbon, $\bar{\tilde{c}}^t$ the
mean carbon outcome realised in season $t$, $\bar{d}^t \in [0,1]^{|B|}$ the
vector of per-practice adoption rates, and $\bar{\rho}^t \in [0,1]$ the
participation rate. All components are normalised to approximately $[0,1]$ for
training stability. Farmer types $\theta_j$ remain unobserved, consistent with the
adverse selection setting.

\paragraph{Action space.} At each season $t$, the aggregator offers a pooled contract
$(\mathbf{a}^t, \alpha^t, R^t)$. Each farmer $j$ responds with a participation decision, 
adoption decision and effort profile.
\begin{equation}
    \label{eqn:seasonal_adoption}
    \begin{split}
    S_j^t = \Big\{ i \in B \;:\;
        e_{ji}^t\Big( &y_{ji}^{t_{ji}} - k_{ji}^{t_{ji}} \\
        &+ c_i \big(1 - \tfrac{C_j^t}{C_m}\big) R^t \Big)
        + a_i^t \;\ge\; 0 \Big\}
    \end{split}
\end{equation}

The farm-level MRV cost and decomposition terms are absent as neither
depends on $i$ and so neither affects the marginal comparison. A farmer accepts the contract only if the expected utility from participation equals or exceeds the baseline, i.e $U^t_{f_j}(S_j^t, \mathbf{e}_j^t) \geq u_{0_j}^t$. Otherwise the farmer retains the outside option by adopting only practices that are inherently profitable without contract incentives.

\paragraph{Farmer utility.}
Farmer $j$'s utility follows
Eq.~\ref{eqn:mh_farmers_utility}, with season-dependent yield, cost and carbon terms defined by the multi-season dynamics of
Eqs.~\ref{eqn:yield_ramp_up}--\ref{eqn:soc_dynamic}, and farmer $j$'s farm
size $f_j$:
\begin{align}
U_{f_j}^t(S_j^t, \mathbf{e}_j^t) = &\sum_{i \in S_j^t} \left( e_{ji}^{t} y_{ji}^{t_{ji}} - e_{ji}^{t} k_{ji}^{t_{ji}} + a_i^t \right) f_j \nonumber \\
&- \alpha^t \, m_j + (C_j^{t+1}-C_j^t) \cdot f_j R^t
\label{eqn:mdp_farmer_utility}
\end{align}

Both parties are paid on the actual net carbon sequestered.

\paragraph{Aggregator reward.}
Unlike the single-season objective (Eq.~\ref{eqn:mh_agg_utility}), the
per-season reward here is computed on net carbon after saturation and
decomposition, and summed across the farmer population. The aggregator's
per-season reward is:
\begin{align}
r^t = \sum_{j=1}^{N} \Bigg[ &f_j (P - R^t)(C_j^{t+1}-C_j^t) - f_j \sum_{i \in S^{t}_j} a_i^t \notag \\
&- (1-\alpha^t)\, m_j \Bigg]
\label{eqn:agg_reward}
\end{align}

\section{Methodology}
All main results are aggregated across 25 random seeds. Our objective is to learn the profit-maximising contract for a given
farmer population rather than to generalise across populations, consistent
with our design in which the aggregator is retrained per region. Each seed
therefore fixes a population instance, and the aggregator is trained and
evaluated on that instance. All confidence intervals are 95\% percentile bootstrap intervals resampled
per seed. All p-values use the Wilcoxon signed-rank test, paired over the 25 seeds for between-condition comparisons, and as a one-sample test against zero for the correlation coefficients.

We train the aggregator with TQC (Truncated Quantile Critics; \citet{tqc}). In a preliminary comparison
over 5 seeds per ablation, TQC matched or outperformed SAC~\citep{sac} across all
conditions, from a statistical tie in the MH-only case to a $\sim$7\%
profit-capture-ratio gain (SD-C). On-policy PPO~\cite{ppo} failed to converge to a competitive policy in our setting even
after relaxing the training curriculum, so we focus on off-policy methods. Exact training details are available in SD-D.

\subsection{Welfare Baseline}

To assess the learned contracts, we normalize the aggregator's realized profit
against a \emph{first-best welfare baseline}: the maximum total welfare
achievable under direct carbon-market participation with myopically best-responding farmers. Comparison against the first-best is
standard in contract theory~\citep{laffont} and algorithmic contract
design~\citep{dutting_2024}. We construct it by setting the carbon payment rate
$R$ equal to the market price $P$ and the MRV cost-share $\alpha = 1$, so the
aggregator earns zero margin and farmers face the full social value of their
actions, equivalent to transacting directly in the carbon market with no
intermediary. Farmers best-respond as elsewhere, accepting if participating
profit exceeds or equals their reservation utility $u_0$. Welfare is the sum of
accepted farmers' surplus above this outside option, counting only value created by the contract.

Since this ceiling is invariant to how surplus is split between aggregator and
farmers, an aggregator whose profit approaches it is capturing nearly all
system surplus, while a large gap reflects either value destroyed by frictions
or retained by farmers. We report the profit-capture ratio, RL reward
over first-best welfare, computed per ablation. We also give the full
decomposition into farmer surplus, aggregator surplus, and deadweight loss in SD-J.

\subsection{Ablations}

\begin{table}[htbp]
\centering
\begin{tabular}{lccc}
\toprule
Condition & Effort $e_i$ & Cost type $\theta_j$ & Farm size $f_j$ \\
\midrule
MH+AS   & $[0.2, 1.0]$ & $\sim U[0,1]$ & $5.5$ \\
MH only & $[0.2, 1.0]$ & $0.5$         & $5.5$ \\
AS only & $1.0$          & $\sim U[0,1]$ & $5.5$ \\
Neither & $1.0$          & $0.5$         & $5.5$ \\
\midrule
MH+AS   & $[0.2, 1.0]$ & $\sim U[0,1]$ & $\sim U[1,10]$ \\
MH only & $[0.2, 1.0]$ & $0.5$         & $\sim U[1,10]$ \\
AS only & $1.0$          & $\sim U[0,1]$ & $\sim U[1,10]$ \\
Neither & $1.0$          & $0.5$         & $\sim U[1,10]$ \\
\bottomrule
\end{tabular}
\caption{Ablations.}
\label{tab:ablation}
\end{table}

We run 8 ablations (Table~\ref{tab:ablation}) by independently toggling adverse selection and moral hazard, each crossed with two farm-size regimes. When moral hazard is absent, effort is fixed at $e_i = 1.0$. When it is present, farmers privately choose $e_i \in [\underline{e}, 1.0]$ per practice, realised at $\underline{e}$ or $1$,  which the aggregator cannot observe. When adverse selection is absent, the cost multiplier is fixed at the population mean $\theta_j = 0.5$. When it is present, $\theta_j$ is drawn independently per farmer and never revealed. Crossing these four combinations with fixed (5.5 ha) and heterogeneous ($f_j \sim U[1,10]$ ha) farm sizes isolates the effect of the information frictions from that of farm-size heterogeneity. Initial SOC is drawn uniformly ($C^0_j \sim U[30,80]$ tCO\textsubscript{2}e/ha) in all ablations.

\section{Results}

\subsection{Fixed Farm Size}
We first fix farm size to isolate the effect of the frictions. 
\begin{table}[htbp]
\centering
\begin{tabular}{lccc}
\toprule
\textbf{Ablation} & \textbf{Mean} & \textbf{Std} & \textbf{95\% CI} \\
\midrule
MH+AS & 0.599 & 0.155 & [0.538, 0.658] \\
MH only & \underline{0.699} & 0.147 & [0.640, 0.755] \\
AS only & 0.566 & \textbf{0.109} & [0.520, 0.607] \\
Neither & \textbf{0.809} & \underline{0.115} & [0.761, 0.851] \\
\bottomrule
\end{tabular}
\caption{Profit-capture ratio, fixed farm size.}
\label{tab:efficiency_ratio}
\end{table}

The profit-capture ratio is calculated as the cumulative aggregator reward over 5 seasons divided by the cumulative welfare baseline.
Table~\ref{tab:efficiency_ratio} shows that the Neither condition outperforms all others. More surprisingly,  MH+AS outperforms AS only,  in aggregate though the difference is not statistically 
significant across seeds ($p = 0.080$). The effect arises from moral hazard, with practices unprofitable at full effort becoming profitable at low effort, raising adoption and ultimately profitability. Because effort scales realised loss but not the action payment, moral hazard lets the farmer limit downside on an otherwise unprofitable practice rather than bear it in full. Aggregated across the population, this effect raises adoption under MH+AS relative to AS only, outweighing the additional friction introduced by moral hazard. We provide a complete worked example of this effect in SD-H.
Table~\ref{tab:adoption_rate} shows how adoption rate varies across conditions, averaged over farmers, seasons and seeds. More information regarding the contract payment structure itself is available in SD-K.

\begin{table}[htbp]
\centering
\begin{tabular}{lccc}
\toprule
\textbf{Ablation} & \textbf{Mean} & \textbf{Std} & \textbf{95\% CI} \\
\midrule
MH+AS & 0.542 & 0.204 & [0.464, 0.623] \\
MH only & \underline{0.592} & 0.179 & [0.521, 0.661] \\
AS only & 0.428 & \textbf{0.093} & [0.392, 0.464] \\
Neither & \textbf{0.691} & \underline{0.154} & [0.630, 0.751] \\
\bottomrule
\end{tabular}
\caption{Adoption rate, fixed farm size.}
\label{tab:adoption_rate}
\end{table}

\subsection{Variable Farm Size}
We next allow farm size to vary across farmers, drawn from a uniform distribution. Real smallholdings are skewed towards smaller
areas and better described by a log-normal distribution, but the uniform draw provides coverage across the size range, which are better suited to characterizing adoption by farm size.

\begin{table*}[htbp]
\centering
\begin{tabular}{lccc}
\toprule
\textbf{Metric} & \textbf{Small [1-4 Ha]} & \textbf{Medium [4-7 Ha]} & \textbf{Large [7-10 Ha]} \\
\midrule
First-Best Adoption & 19.4\% [16.6\%, 22.3\%] & 68.7\% [65.1\%, 72.5\%] & 88.1\% [86.4\%, 89.7\%] \\
RL Adoption & 1.7\% [0.8\%, 2.8\%] & 16.5\% [13.2\%, 20.4\%] & 77.3\% [71.9\%, 82.7\%] \\
RL Adoption / FB Adoption & 8.2\% [4.3\%, 12.7\%] & 23.8\% [19.5\%, 28.6\%] & 87.7\% [81.7\%, 93.4\%] \\
Expected RL Profit/Farm   & \$7 [\$3, \$11] & \$96 [\$82, \$110] & \$404 [\$375, \$434] \\
Expected per-farm RL Profit/ha  & \$2 [\$1, \$3] & \$16 [\$14, \$19] & \$47 [\$44, \$50] \\
\bottomrule
\end{tabular}
\caption{Adoption and profit by farm size tercile: (MH+AS), $m_0=300, m_1=80, \delta=0.7$ (95\% CI).}
\label{tab:farm_size_adoption}
\end{table*}

\begin{table*}[htbp]
\centering
\begin{tabular}{lccc}
\toprule
\textbf{Metric} & \textbf{Small [1-4 Ha]} & \textbf{Medium [4-7 Ha]} & \textbf{Large [7-10 Ha]} \\
\midrule
First-Best Adoption & 72.1\% [69.0\%, 75.4\%] & 72.1\% [68.2\%, 76.0\%] & 71.3\% [68.9\%, 73.9\%] \\
RL Adoption & 56.5\% [47.2\%, 65.9\%] & 57.4\% [48.8\%, 66.0\%] & 56.1\% [47.9\%, 64.2\%] \\
RL Adoption / FB Adoption & 78.8\% [65.8\%, 92.1\%] & 78.3\% [68.3\%, 87.9\%] & 77.6\% [67.3\%, 87.6\%] \\
Expected RL Profit/Farm   & \$65 [\$54, \$77] & \$157 [\$137, \$177] & \$233 [\$203, \$264] \\
Expected per-farm RL Profit/ha & \$26 [\$22, \$30] & \$28 [\$25, \$32] & \$27 [\$24, \$31] \\
\bottomrule
\end{tabular}
\caption{Adoption and profit by farm size tercile: (MH+AS), $m_0=0, m_1=101.3, \delta=1.0$ (95\% CI).}
\label{tab:farm_size_adoption_no_fixed}
\end{table*}

\begin{table}[htbp]
\centering
\begin{tabular}{lccc}
\toprule
\textbf{Ablation} & \textbf{Mean} & \textbf{Std} & \textbf{95\% CI} \\
\midrule
MH+AS & \underline{0.620} & 0.109 & [0.577, 0.663] \\
MH only & 0.612 & \textbf{0.096} & [0.574, 0.650] \\
AS only & 0.568 & 0.104 & [0.528, 0.610] \\
Neither & \textbf{0.680} & \underline{0.100} & [0.641, 0.720] \\
\bottomrule
\end{tabular}
\caption{Profit-capture ratio, variable farm size.}
\label{tab:efficiency_ratio_het}
\end{table}

\begin{table}[htbp]
\centering
\begin{tabular}{lccc}
\toprule
\textbf{Ablation} & \textbf{Mean} & \textbf{Std} & \textbf{95\% CI} \\
\midrule
MH+AS & 0.320 & \underline{0.075} & [0.292, 0.350] \\
MH only & 0.315 & \textbf{0.046} & [0.298, 0.333] \\
AS only & \underline{0.337} & 0.076 & [0.310, 0.369] \\
Neither & \textbf{0.411} & 0.095 & [0.376, 0.449] \\
\bottomrule
\end{tabular}
\caption{Adoption rate, variable farm size.}
\label{tab:adoption_rate_het}
\end{table}

\paragraph{Friction comparisons.} Profit-capture ratios (Table~\ref{tab:efficiency_ratio_het}) follow a similar overall ordering, but the gap between the best and worst performers narrows, and the best performer does substantially worse than under fixed farm size. We find two interesting results. First, MH+AS again outperforms AS only, this time significantly across seeds ($p < 10^{-4}$). This is not driven by adoption rate differences (Table~\ref{tab:adoption_rate_het}), but by a second effect of moral hazard. The aggregator learns to offer a lower average action payment, as it no longer has to compensate for the full cost of effort for otherwise unprofitable actions. We calculated the average payment across adopted practices, and it was \$1.35/ha less ($p < 10^{-4}$) in the MH+AS case compared with the AS only case.

Second, MH+AS achieves comparable performance to MH-only ($p = 0.916$), suggesting the adverse selection mechanism does not meaningfully reduce performance when moral hazard is present. We leave full characterization of this interaction to future work.

\paragraph{Amplified smallholder exclusion.}Beyond the contracts themselves, our framework speaks to who participates and
why. Table~\ref{tab:farm_size_adoption} reports adoption by farm-size tercile
under RL and under first-best, together with the RAS. Two
gradients emerge, with the first mechanical. Fixed per-farmer MRV costs do
not scale with area, so the same cost is amortised over more
hectares on larger farms, making larger farms more likely to clear the threshold. The first-best already embeds this gradient because it uses the same MRV cost structure. RAS captures the residual after this control. The aggregator
realises 87.7\% of achievable adoption on large farms but only 8.2\% on
smallholdings. Because the benchmark already incorporates fixed costs, this residual exceeds the gradient attributable to cost structure alone.

We attribute the residual to the interaction between profit maximisation and the
pooling constraint. Under a single per-hectare schedule, any increase in payment generous
enough to bring marginal smallholders above their threshold must also be
paid on every hectare already enrolled. That inframarginal cost scales
with the area of the large farms in the pool and outweighs the thin
margin recovered from newly included smallholders, so the
profit-maximising contract stays below the level that would admit them.

\paragraph{MRV counterfactual.} To test whether this exclusion is rooted in the cost structure rather than a
fundamental barrier, we consider a counterfactual in which MRV costs are made
purely area-proportional. The fixed per-farmer component is removed, the scaling exponent is set to one, and the variable rate is recalibrated so that total MRV
liability over the farmer population is unchanged from baseline (SD-I). The regime therefore \emph{redistributes} MRV cost across farm sizes rather than lowering it. This requires large farmers in the pool. The same effect could also be reached through alternate strategies, such as government or NGO socialisation, or by lowering fixed costs through cheaper monitoring methods such as remote sensing. Evaluating these alternatives is exactly the kind of policy question the framework is
built for.

Table~\ref{tab:farm_size_adoption_no_fixed} shows that the size gradient
vanishes. First-best and RL adoption are both flat across terciles, as is RAS,
whose intervals fully overlap. Expected profit per hectare\footnote{ Let $ s_j $ denote the aggregator's profit from farm $j$ and $f_j$ its area. Expectations are taken over all farms in the tercile. $\mathbb{E}[ s_j/f_j ]$ is the expected per-farm profit per hectare.} is now flat across
farm sizes, removing the incentive to favour larger holdings that the fixed-cost
component otherwise creates. The equalisation is two-sided, as smallholder
adoption rises sharply while large-farm adoption falls, the latter reflecting
the higher share of MRV costs large farms now bear. The change does not reduce overall participation or enrolled area, as both increase under the RL contract, so smallholder gains outweigh large-farm losses on both counts
(SD-I). Aggregator profit falls by roughly 10\%, which is the
expected consequence of removing a cost asymmetry exploited by the profit-maximising contract. The restructuring is therefore not one an aggregator would adopt
unilaterally, which is what makes it a policy instrument rather than a business
recommendation.

\paragraph{Farmer-level correlates.}Table~\ref{tab:adoption_correlations_het} corroborates this pattern at the farmer level. Across all information conditions, farm
size has the strongest association with participation, with correlations that are statistically significant. Correlations with initial SOC and adoption cost are weaker, despite generally being significant. The substantive
point is not that larger farms participate more, a generic consequence
of area-scaled contracts, but that the effect is large enough to mask the
informational frictions, which surface only once farm size is controlled for.
\begin{table}[htbp]
\centering
\begin{tabular}{lccc}
\toprule
\textbf{Condition} & \textbf{Mean $r$} & \textbf{95\% CI} & \textbf{$p$-value} \\
\midrule
\multicolumn{4}{l}{\textbf{Farm Size}} \\
MH+AS & 0.693 & [0.658, 0.726] & $< 10^{-4}$ \\
MH only & 0.748 & [0.723, 0.771] & $< 10^{-4}$ \\
AS only & 0.642 & [0.622, 0.661] & $< 10^{-4}$ \\
Neither & 0.787 & [0.771, 0.803] & $< 10^{-4}$ \\
\midrule
\multicolumn{4}{l}{\textbf{Initial SOC}} \\
MH+AS & -0.159 & [-0.191, -0.125] & $< 10^{-4}$ \\
MH only & -0.209 & [-0.252, -0.166] & $< 10^{-4}$ \\
AS only & -0.114 & [-0.151, -0.073] & $< 10^{-4}$ \\
Neither & -0.155 & [-0.199, -0.108] & $< 10^{-4}$ \\
\midrule
\multicolumn{4}{l}{\textbf{Adoption Cost}} \\
MH+AS & -0.114 & [-0.155, -0.075] & $< 10^{-4}$ \\
MH only & --- & --- & --- \\
AS only & 0.033 & [-0.001, 0.071] & 0.1336 \\
Neither & --- & --- & --- \\
\bottomrule
\end{tabular}
\caption{Point-biserial correlation between farmer characteristics and adoption, variable farm size.}
\label{tab:adoption_correlations_het}
\end{table}

\section{Limitations}

\paragraph{Additionality.} We do not impose an additionality constraint in our model. However, as Proposition 2 (SD-B) demonstrates, this yields a conservative lower bound on smallholder exclusion because enforcing additionality strictly raises the minimum viable farm size. We also abstract away from carbon permanence, which real-world registries enforce over multi-decade horizons rather than our five-season simulation.

\paragraph{Farmer behaviour.}Farmers best-respond each season, as this is the standard assumption in learning-based Stackelberg settings~\citep{nika_22}, and the short effective planning horizons documented among smallholders~\citep{duflo} make the asymmetry with our forward-looking aggregator realistic rather than a modelling shortcut. Farmers are also risk-neutral and face no credit, tenure, or participation costs beyond MRV. These omitted frictions bind harder on smallholders, so our exclusion estimates are a lower bound. 

\paragraph{Non-learning baselines.}We do not benchmark against non-learning optimization algorithms, as we make no claims regarding the absolute efficiency of the learner. Such a comparison would ask whether the size gradient comes from the environment or from the learner, and our cost counterfactual answers that directly.

\paragraph{Calibration scope.}Finally, our parameters characterise a single regional instantiation with farm sizes drawn uniformly and MRV charged per farm. The magnitudes should be
read as illustrative of that setting, and thorough empirical validation is left to future work. The simulator is configurable, so a user can
recalibrate to their own region and test with no changes to the model.

\section{Conclusion}

We have studied how commercially rational carbon-farming contracts shape participation under adverse selection, moral hazard, and multi-season agronomic dynamics. Our results show that commercially rational pooled contracts systematically exclude smallholders beyond what the underlying MRV cost structure alone explains, concentrating participation among larger farms. This exclusion is not inevitable. Redistributing MRV costs in proportion to cultivated area, while preserving the overall expected MRV liability of the system, removes the participation gradient. Greater smallholder inclusion directly improves both equity and the share of soil carbon potential that markets can reach.

Beyond climate finance, this formulation contributes directly to the growing subfield of automated contract design. We model the problem as a POMDP in which a learning principal offers a
single contract to a heterogeneous, partially observed population under
non-stationary dynamics. We treat moral hazard and adverse selection
jointly and over a sequential horizon whereas prior learning approaches
address the frictions in isolation and in stationary settings.

While learned contracts are traditionally evaluated by principal utility, welfare, or regret, we demonstrate the necessity of quantifying their distributional consequences in public-good domains. The proposed Realised Adoption Share (RAS) distinguishes exclusion imposed by baseline cost structures from that amplified by profit maximisation under pooling. This dynamic extends beyond agriculture. Any pooled contract with fixed per-participant costs, such as microcredit or insurance, excludes those below a specific value cutoff. Our results demonstrate that profit maximisation pushes this cutoff well above the welfare-maximising baseline, a margin RAS is built to measure. More broadly, our work illustrates how AI can support the design and evaluation of economic mechanisms by revealing their societal consequences before deployment.

We release our open-source environment to serve both as an experimental testbed for AI researchers and as a decision-support tool for policymakers evaluating market interventions.

\bibliography{references}

\begin{thebibliography}{43}
\providecommand{\natexlab}[1]{#1}

\bibitem[{Antle et~al.(2003)Antle, Capalbo, Mooney, Elliott, and
  Paustian}]{antle2003}
Antle, J.; Capalbo, S.; Mooney, S.; Elliott, E.; and Paustian, K. 2003.
\newblock Spatial heterogeneity, contract design, and the efficiency of carbon
  sequestration policies for agriculture.
\newblock \emph{Journal of Environmental Economics and Management}, 46(2):
  231--250.

\bibitem[{Barbato and Strong(2023)}]{Barbato2023FarmerPerspectives}
Barbato, C.~T.; and Strong, A.~L. 2023.
\newblock Farmer Perspectives on Carbon Markets Incentivizing Agricultural Soil
  Carbon Sequestration.
\newblock \emph{npj Climate Action}, 2: 26.

\bibitem[{Bellassen et~al.(2015)Bellassen, Stephan, Afriat, Alberola, Barker,
  Chang, Chiquet, Cochran, Deheza, Dimopoulos, Foucherot, Jacquier, Morel,
  Robinson, and Shishlov}]{scale}
Bellassen, V.; Stephan, N.; Afriat, M.; Alberola, E.; Barker, A.; Chang, J.-P.;
  Chiquet, C.; Cochran, I.; Deheza, M.; Dimopoulos, C.; Foucherot, C.;
  Jacquier, G.; Morel, R.; Robinson, R.; and Shishlov, I. 2015.
\newblock Monitoring, Reporting and Verifying Emissions in the Climate Economy.
\newblock \emph{Nature Climate Change}, 5(4): 319--328.

\bibitem[{Bollini et~al.(2026)Bollini, Bacchiocchi, Coutts, Castiglioni, and
  Marchesi}]{bollini2026learning}
Bollini, M.; Bacchiocchi, F.; Coutts, S.; Castiglioni, M.; and Marchesi, A.
  2026.
\newblock Learning in Bayesian Stackelberg Games With Unknown Follower's Types.
\newblock In \emph{Forty-third International Conference on Machine Learning}.

\bibitem[{Cariappa and Krishna(2025)}]{Cariappa2025}
Cariappa, A. A.~G.; and Krishna, V.~V. 2025.
\newblock Carbon farming in India: are the existing projects inclusive,
  additional, and permanent?
\newblock \emph{Climate Policy}, 25(5): 756--771.

\bibitem[{Coleman and Jenkinson(1996)}]{coleman1996}
Coleman, K.; and Jenkinson, D.~S. 1996.
\newblock RothC-26.3 - A Model for the turnover of carbon in soil.
\newblock In Powlson, D.~S.; Smith, P.; and Smith, J.~U., eds.,
  \emph{Evaluation of Soil Organic Matter Models}, 237--246. Berlin,
  Heidelberg: Springer Berlin Heidelberg.
\newblock ISBN 978-3-642-61094-3.

\bibitem[{Conitzer and Sandholm(2003)}]{conitzer}
Conitzer, V.; and Sandholm, T. 2003.
\newblock Automated mechanism design: complexity results stemming from the
  single-agent setting.
\newblock In \emph{Proceedings of the 5th International Conference on
  Electronic Commerce}, ICEC '03, 17–24. New York, NY, USA: Association for
  Computing Machinery.
\newblock ISBN 1581137885.

\bibitem[{Dabbert and Madden(1986)}]{dabbert1986}
Dabbert, S.; and Madden, P. 1986.
\newblock The transition to organic agriculture: A multi-year simulation model
  of a Pennsylvania farm.
\newblock \emph{American Journal of Alternative Agriculture}, 1(3): 99--107.

\bibitem[{Dahlgreen and Parr(2024)}]{agroparam1}
Dahlgreen, J.; and Parr, A. 2024.
\newblock Exploring the Impact of Alternate Wetting and Drying and the System
  of Rice Intensification on Greenhouse Gas Emissions: A Review of Rice
  Cultivation Practices.
\newblock \emph{Agronomy}, 14(2).

\bibitem[{Dai~Li, Immorlica, and Lucier(2022)}]{dai_li}
Dai~Li, W.; Immorlica, N.; and Lucier, B. 2022.
\newblock Contract {{Design}} for {{Afforestation Programs}}.
\newblock In Feldman, M.; Fu, H.; and {Talgam-Cohen}, I., eds., \emph{Web and
  {{Internet Economics}}}, volume 13112, 113--130. Cham: Springer International
  Publishing.
\newblock ISBN 978-3-030-94675-3 978-3-030-94676-0.

\bibitem[{Das et~al.(2025)Das, Chatterjee, Saha, Sarkar, Alam, Dey, Ghosh,
  Nayak, Smith, and Pathak}]{agroparam2}
Das, S.~R.; Chatterjee, D.; Saha, S.; Sarkar, D.; Alam, R.; Dey, S.; Ghosh, S.;
  Nayak, B.~K.; Smith, P.; and Pathak, H. 2025.
\newblock Enhancing carbon sequestration potential of lowland rice
  agroecosystems for environmentally clean production system: A review.
\newblock \emph{Climate Smart Agriculture}, 2(2): 100054.

\bibitem[{Datta et~al.(2025)Datta, Wilke, Charles, Hasenick, Ulbrich, Singh,
  Sears, and Robertson}]{datta2025}
Datta, A.; Wilke, B.; Charles, C.; Hasenick, M.; Ulbrich, T.; Singh, M.; Sears,
  M.; and Robertson, G.~P. 2025.
\newblock Crop performance and profitability for the initial transition years
  of a regenerative cropping system in the Upper Midwest United States.
\newblock \emph{Journal of Environmental Quality}, 54(6): 1572--1585.

\bibitem[{Ducos, Dupraz, and Bonnieux(2009)}]{ducos2009}
Ducos, G.; Dupraz, P.; and Bonnieux, F. 2009.
\newblock Agri-environment contract adoption under fixed and variable
  compliance costs.
\newblock \emph{Journal of Environmental Planning and Management}, 52(5):
  669--687.

\bibitem[{Duflo, Kremer, and Robinson(2011)}]{duflo}
Duflo, E.; Kremer, M.; and Robinson, J. 2011.
\newblock Nudging Farmers to Use Fertilizer: Theory and Experimental Evidence
  from Kenya.
\newblock \emph{American Economic Review}, 101(6): 2350–90.

\bibitem[{D\"{u}tting, Roughgarden, and Talgam-Cohen(2019)}]{dutting_2019}
D\"{u}tting, P.; Roughgarden, T.; and Talgam-Cohen, I. 2019.
\newblock Simple versus Optimal Contracts.
\newblock In \emph{Proceedings of the 2019 ACM Conference on Economics and
  Computation}, EC '19, 369–387. New York, NY, USA: Association for Computing
  Machinery.
\newblock ISBN 9781450367929.

\bibitem[{Dütting, Feldman, and Talgam-Cohen(2024)}]{dutting_2024}
Dütting, P.; Feldman, M.; and Talgam-Cohen, I. 2024.
\newblock Algorithmic Contract Theory: A Survey.
\newblock \emph{Foundations and Trends® in Theoretical Computer Science}, 16:
  211--411.

\bibitem[{Eagle, Uludere~Aragon, and Gordon(2022)}]{eagle2022realizable}
Eagle, A.; Uludere~Aragon, N.; and Gordon, D. 2022.
\newblock The Realizable Magnitude of Carbon Sequestration in Global Cropland
  Soils: Socioeconomic Factors.
\newblock Technical report, Environmental Defense Fund, New York, New York.

\bibitem[{Foster and Rosenzweig(1995)}]{foster1995}
Foster, A.~D.; and Rosenzweig, M.~R. 1995.
\newblock Learning by {{Doing}} and {{Learning}} from {{Others}}: {{Human
  Capital}} and {{Technical Change}} in {{Agriculture}}.
\newblock \emph{Journal of Political Economy}, 103(6): 1176--1209.

\bibitem[{Gerstgrasser and Parkes(2023)}]{gerstgrasser2023}
Gerstgrasser, M.; and Parkes, D.~C. 2023.
\newblock Oracles \& Followers: Stackelberg Equilibria in Deep Multi-Agent
  Reinforcement Learning.
\newblock In Krause, A.; Brunskill, E.; Cho, K.; Engelhardt, B.; Sabato, S.;
  and Scarlett, J., eds., \emph{Proceedings of the 40th International
  Conference on Machine Learning}, volume 202 of \emph{Proceedings of Machine
  Learning Research}, 11213--11236. PMLR.

\bibitem[{Guruganesh, Schneider, and Wang(2021)}]{guruganesh_2021}
Guruganesh, G.; Schneider, J.; and Wang, J.~R. 2021.
\newblock Contracts under Moral Hazard and Adverse Selection.
\newblock In \emph{Proceedings of the 22nd ACM Conference on Economics and
  Computation}, EC '21, 563–582. New York, NY, USA: Association for Computing
  Machinery.
\newblock ISBN 9781450385541.

\bibitem[{Haarnoja et~al.(2018)Haarnoja, Zhou, Abbeel, and Levine}]{sac}
Haarnoja, T.; Zhou, A.; Abbeel, P.; and Levine, S. 2018.
\newblock Soft Actor-Critic: Off-Policy Maximum Entropy Deep Reinforcement
  Learning with a Stochastic Actor.
\newblock In Dy, J.; and Krause, A., eds., \emph{Proceedings of the 35th
  International Conference on Machine Learning}, volume~80 of \emph{Proceedings
  of Machine Learning Research}, 1861--1870. PMLR.

\bibitem[{Haghtalab et~al.(2022)Haghtalab, Lykouris, Nietert, and
  Wei}]{nika_22}
Haghtalab, N.; Lykouris, T.; Nietert, S.; and Wei, A. 2022.
\newblock Learning in Stackelberg Games with Non-myopic Agents.
\newblock In \emph{Proceedings of the 23rd ACM Conference on Economics and
  Computation}, EC '22, 917–918. New York, NY, USA: Association for Computing
  Machinery.
\newblock ISBN 9781450391504.

\bibitem[{Hart and Latacz-Lohmann(2005)}]{rob_2005}
Hart, R.; and Latacz-Lohmann, U. 2005.
\newblock Combating moral hazard in agri-environmental schemes: a
  multiple-agent approach.
\newblock \emph{European Review of Agricultural Economics}, 32(1): 75--91.

\bibitem[{Ivanov et~al.(2024)Ivanov, Dütting, Talgam-Cohen, Wang, and
  Parkes}]{ivanov2024}
Ivanov, D.; Dütting, P.; Talgam-Cohen, I.; Wang, T.; and Parkes, D.~C. 2024.
\newblock Principal-Agent Reinforcement Learning: Orchestrating AI Agents with
  Contracts.
\newblock arXiv:2407.18074.

\bibitem[{Johansson, Andersson, and Fischer(2025)}]{Johansson2025}
Johansson, E.; Andersson, E.; and Fischer, K. 2025.
\newblock The race for carbon in farmland: global mapping of the emerging
  voluntary market for soil carbon credits.
\newblock \emph{International Journal of Sustainable Development \& World
  Ecology}, 32(7): 810--826.

\bibitem[{Kuznetsov et~al.(2020)Kuznetsov, Shvechikov, Grishin, and
  Vetrov}]{tqc}
Kuznetsov, A.; Shvechikov, P.; Grishin, A.; and Vetrov, D. 2020.
\newblock Controlling Overestimation Bias with Truncated Mixture of Continuous
  Distributional Quantile Critics.
\newblock In III, H.~D.; and Singh, A., eds., \emph{Proceedings of the 37th
  International Conference on Machine Learning}, volume 119 of
  \emph{Proceedings of Machine Learning Research}, 5556--5566. PMLR.

\bibitem[{Laffont and Martimort(2002)}]{laffont}
Laffont, J.-J.; and Martimort, D. 2002.
\newblock \emph{The Theory of Incentives: The Principal-Agent Model}.
\newblock Princeton University Press.
\newblock ISBN 9780691091846.

\bibitem[{Lal(2011)}]{LAL2011}
Lal, R. 2011.
\newblock Sequestering carbon in soils of agro-ecosystems.
\newblock \emph{Food Policy}, 36: S33--S39.

\bibitem[{Lal(2015)}]{Lal_2015}
Lal, R. 2015.
\newblock Soil Carbon Sequestration in Agroecosystems of India.
\newblock \emph{Journal of the Indian Society of Soil Science}, 63(2).

\bibitem[{Lowder et~al.(2025)Lowder, Arslan, Cabrera~Cevallos, O'Neill, and
  De~La O~Campos}]{lowder_global_2025}
Lowder, S.; Arslan, A.; Cabrera~Cevallos, C.~E.; O'Neill, M.; and De~La
  O~Campos, A.~P. 2025.
\newblock A global update on the number of farms, farm size and farmland
  distribution – {Background} paper for {The} {State} of {Food} and
  {Agriculture} 2025.
\newblock {FAO} {Agricultural} {Development} {Economics} {Working} {Paper}
  25-14, FAO, Rome.

\bibitem[{MacKenzie, Ohndorf, and Palmer(2012)}]{ian_2012}
MacKenzie, I.~A.; Ohndorf, M.; and Palmer, C. 2012.
\newblock Enforcement-proof contracts with moral hazard in precaution: ensuring
  'permanence' in carbon sequestration.
\newblock \emph{Oxford Economic Papers}, 64(2): 350--374.

\bibitem[{Priyanka et~al.(2025)Priyanka, Charan, Suresh, Sunkara, Patil, Sagar,
  Trivedi, Soumya, Paul, Hadimani, Babu, Trivedi, and
  Narahari}]{priyanka2025carbonfarming}
Priyanka, V.; Charan, G.; Suresh, R.~P.; Sunkara, T.; Patil, M.; Sagar, K.;
  Trivedi, A.; Soumya, K.; Paul, S.; Hadimani, P.; Babu, G.; Trivedi, R.; and
  Narahari, Y. 2025.
\newblock Carbon {{Farming}}: {{An Expository}}, {{Inter-Disciplinary Survey}}.
\newblock \emph{Journal of the Indian Institute of Science}, 105(2-3):
  337--399.

\bibitem[{Raffin et~al.(2021)Raffin, Hill, Gleave, Kanervisto, Ernestus, and
  Dormann}]{stable-baselines3}
Raffin, A.; Hill, A.; Gleave, A.; Kanervisto, A.; Ernestus, M.; and Dormann, N.
  2021.
\newblock Stable-Baselines3: Reliable Reinforcement Learning Implementations.
\newblock \emph{Journal of Machine Learning Research}, 22(268): 1--8.

\bibitem[{Raina, Zavalloni, and Viaggi(2024)}]{Raina2024}
Raina, N.; Zavalloni, M.; and Viaggi, D. 2024.
\newblock Incentive mechanisms of carbon farming contracts: A systematic
  mapping study.
\newblock \emph{Journal of Environmental Management}, 352: 120126.

\bibitem[{Ricciardi et~al.(2018)Ricciardi, Ramankutty, Mehrabi, Jarvis, and
  Chookolingo}]{Ricciardi}
Ricciardi, V.; Ramankutty, N.; Mehrabi, Z.; Jarvis, L.; and Chookolingo, B.
  2018.
\newblock How much of the world's food do smallholders produce?
\newblock \emph{Global Food Security}, 17: 64--72.

\bibitem[{Scheid et~al.(2025)Scheid, Boursier, Durmus, Moulines, and
  Jordan}]{scheid2025}
Scheid, A.; Boursier, E.; Durmus, A.; Moulines, E.; and Jordan, M.~I. 2025.
\newblock Online Decision-Making in Tree-Like Multi-Agent Games with Transfers.
\newblock arXiv:2501.19388.

\bibitem[{Schulman et~al.(2017)Schulman, Wolski, Dhariwal, Radford, and
  Klimov}]{ppo}
Schulman, J.; Wolski, F.; Dhariwal, P.; Radford, A.; and Klimov, O. 2017.
\newblock Proximal Policy Optimization Algorithms.
\newblock arXiv:1707.06347.

\bibitem[{Stewart et~al.(2007)Stewart, Paustian, Conant, Plante, and
  Six}]{stewart2007}
Stewart, C.; Paustian, K.; Conant, R.; Plante, A.; and Six, J. 2007.
\newblock Soil carbon saturation: Concept, evidence and evaluation.
\newblock \emph{Biogeochemistry}, 86: 19--31.

\bibitem[{Tyagi and Haritash(2025)}]{agroparam3}
Tyagi, A.; and Haritash, A.~K. 2025.
\newblock Climate-smart agriculture, enhanced agroproduction, and carbon
  sequestration potential of agroecosystems in India: a meta-analysis.
\newblock \emph{Journal of Environmental Studies and Sciences}, 15: 167--185.

\bibitem[{Wang et~al.(2024)Wang, Li, Zha, and Wang}]{wang_2024}
Wang, H.; Li, W.; Zha, H.; and Wang, B. 2024.
\newblock Carbon Market Simulation with Adaptive Mechanism Design.
\newblock In Larson, K., ed., \emph{Proceedings of the Thirty-Third
  International Joint Conference on Artificial Intelligence, {IJCAI-24}},
  8824--8828. International Joint Conferences on Artificial Intelligence
  Organization.
\newblock Demo Track.

\bibitem[{Welsh, Grover, and Jaimungal(2026)}]{welsh2026}
Welsh, L.; Grover, U.; and Jaimungal, S. 2026.
\newblock Multi-Agent Reinforcement Learning for Greenhouse Gas Offset Credit
  Markets.
\newblock arXiv:2504.11258.

\bibitem[{{World Bank}(2021)}]{agroparam4}
{World Bank}. 2021.
\newblock \emph{Soil Organic Carbon MRV Sourcebook for Agricultural
  Landscapes}.
\newblock Washington, DC: World Bank.

\bibitem[{Wu et~al.(2024)Wu, Chen, Wang, Wang, and Xu}]{wu2024}
Wu, J.; Chen, S.; Wang, M.; Wang, H.; and Xu, H. 2024.
\newblock Contractual Reinforcement Learning: Pulling Arms with Invisible
  Hands.
\newblock arXiv:2407.01458.

\end{thebibliography}

\appendix
\setcounter{proposition}{0}
\setcounter{secnumdepth}{1} 
\section{IR Constraint}

\begin{proposition}
At any solution to the aggregator's problem in which the adoption
constraint is slack, the Individual Rationality constraint binds:
equilibrium farmer utility satisfies $U_f^* = u_0$.
\end{proposition}
\begin{proof}
Suppose $U_f > u_0$ at an optimum $(\mathbf{a}^*, \alpha^*)$ with a
participating farmer. If $\mathbf{a}^* = \mathbf{0}$ then $S = S_0$, so
$U_f = u_0 -\alpha^* m \leq u_0$ by Eq.~2, a contradiction. Hence
$a_i^* > 0$ for some $i \in S$. Let
$\eta = \min_{i \in S} (y_i - k_i + a_i^*) > 0$ denote the slack in the
adoption constraint and reduce $a_i^*$ by
$\varepsilon < \min\{\eta,\, a_i^*,\, (U_f - u_0)/f\}$. The first bound
preserves $S$, the second keeps payments non-negative, and the third
preserves $U_f \geq u_0$. Since $U_A$ is strictly decreasing in $a_i$,
this strictly increases the aggregator's payoff, contradicting
optimality.
\end{proof}

\section{Equilibrium under Additionality}
\label{app:equilibrium_additionality}
Under an additionality constraint, the aggregator earns carbon credits only for practices not adopted in the farmer's outside option ($S \setminus S_0$), and only offers action payments for the same. The binding IR constraint reduces to:
\begin{equation}
    f \sum_{i \in S \setminus S_0} a_i
    = \alpha \cdot m - f \sum_{i \in S \setminus S_0} (y_i - k_i)
\end{equation}
and the aggregator's equilibrium utility is:
\begin{equation}
    U_A^* = P \cdot f \sum_{i \in S \setminus S_0} c_i
          + f \sum_{i \in S \setminus S_0} (y_i - k_i) - m
\end{equation}
The problem reduces entirely to contracting over additional practices; the outside option determines $S_0$ but drops out of the surplus calculation.

Comparing this to the unconstrained equilibrium utility in Eq.~5 gives
\begin{equation}
    \label{eqn:u_add}
    U_A^{*,\text{no-add}} - U_A^{*,\text{add}}
    = P \cdot f \sum_{i \in S_0} c_i
    \;\geq\; 0
\end{equation}
where non-negativity follows from $c_i \geq 0$. The inequality is
strict whenever $\sum_{i \in S_0} c_i > 0$, that is, whenever the
outside option already contains at least one practice with positive
sequestration. Since every practice in our setting satisfies
$c_i > 0$, this reduces to $S_0 \neq \emptyset$.

Additionality therefore reduces aggregator surplus by the carbon value of
outside-option practices while leaving the per-farmer cost $m(f)$ unchanged.
Since viability trades per-hectare surplus against a cost with a fixed
component, the smallest farm size at which the aggregator breaks even is
weakly larger under additionality.

\begin{proposition}
Let $f^*$ denote the smallest farm size at which $U_A^* \geq 0$, under
$m(f) = m_0 + m_1 f^{\delta}$ with $\delta < 1$. Then
$f^*_{\mathrm{add}} \geq f^*_{\mathrm{no\text{-}add}}$, with strict
inequality whenever $\sum_{i \in S_0} c_i > 0$.
\end{proposition}

\begin{proof}
Immediate from Eq.~\ref{eqn:u_add}. The wedge
$P \sum_{i \in S_0} c_i$ lowers per-hectare surplus while leaving $m(f)$
unchanged, so $U_A^{*,\mathrm{add}}(f) \leq U_A^{*,\mathrm{no\text{-}add}}(f)$
for every $f$, and strictly so when $\sum_{i \in S_0} c_i > 0$. Every farm
size viable under additionality is therefore viable without it.
\end{proof}

An additionality requirement thus raises the minimum viable farm size and
strengthens the exclusion gradient reported in the main text. Our adoption
results should be read as favourable to smallholder participation relative
to a setting with additionality enforced.

\section{Algorithm Comparison}
\begin{table}[htbp]
\centering
\caption{Cumulative profit-capture ratio: TQC vs SAC (mean across 5 seeds), with the relative improvement of TQC over SAC.}
\label{tab:tqc_vs_sac}
\begin{tabular}{lccc}
\toprule
\textbf{Ablation} & \textbf{TQC} & \textbf{SAC} & \textbf{Improvement (\%)} \\
\midrule
MH+AS & 0.590 & 0.561 & +5.1 \\
MH only & 0.679 & 0.682 & -0.4 \\
AS only & 0.555 & 0.519 & +6.9 \\
Neither & 0.807 & 0.783 & +3.0 \\
\bottomrule
\end{tabular}
\end{table}

TQC and SAC were trained with the final hyperparameters as reported in Table~\ref{tab:hyperparams}.
\section{Model Training}
\label{app:model_training}
Training TQC was done via Stable-Baselines3~\citep{stable-baselines3}. Training proceeds through a 4-phase curriculum (Table~\ref{tab:curriculum}) that progressively lengthens the season horizon. The entropy coefficient, set to "auto",  is reset in each phase to encourage exploration. All experiments are repeated independently over 25 random seeds.

In each run, the model is trained and evaluated using the same random seed, and reported results are aggregated across the 25 runs. During training, a callback evaluates the current policy every 500 gradient steps by running a deterministic episode on a fixed evaluation seed and records the cumulative reward. The checkpoint is saved whenever this reward exceeds the previous best, so all reported results use the best-performing checkpoint rather than the final iterate, which avoids performance degradation from late-training instability. Selecting a checkpoint by its score on a single noise realization and then reporting that same score would bias the reported value upward, since the selected checkpoint is the one whose score benefited most from that particular draw. We therefore separate selection from reporting. Every reported number comes from re-evaluating the selected checkpoint on an independent held-out noise draw, holding the farmer population fixed, so the reported values are free of selection bias.

\begin{table}[htbp]
\centering
\begin{tabular}{cccc}
\toprule
Phase & Timesteps & Seasons & Rejection penalty \\
\midrule
1 & 25,000  & 1 & 0.3 \\
2 & 25,000  & 3 & 0.1 \\
3 & 25,000  & 5 & 0.1 \\
4 & 100,000 & 5 & 0.0 \\
\bottomrule
\end{tabular}
\caption{Four-phase training curriculum: horizon and rejection penalty schedule.}
\label{tab:curriculum}
\end{table}

\begin{table}[htbp]
\centering
\setlength{\tabcolsep}{1mm}
\begin{tabular}{ll}
\toprule
Hyperparameter & Value \\
\midrule
Algorithm & TQC (Truncated Quantile Critics)\\
Network (actor, critic) & $(256, 256, 256)$ \\
Learning rate & $3\times10^{-4}$ \\
Batch size & 1024 \\
Buffer size & 500,000 \\
Discount factor $\gamma$ & 1.0 (undiscounted) \\
Entropy coefficient & auto (self-tuning, reset per phase) \\
Learning Starts     & 5000 \\
Parallel environments & 24 \\
Gradient steps & 2 \\
Training timesteps & 175k \\
\bottomrule
\end{tabular}
\caption{TQC hyperparameters.}
\label{tab:hyperparams}
\end{table}

During training, the aggregator's reward included an additional rejection
penalty term, subtracted per farmer who did not adopt, following the schedule
in Table~\ref{tab:curriculum}. This penalty is reduced to zero by the final
training phase, so all reported results reflect the aggregator's true
objective (Eq.17). If any additional details regarding the training setup are required, please refer to the Supplementary code.

\section{Agronomic Parameters}

\begin{table}[htbp]
\centering
\setlength{\tabcolsep}{1mm}
\caption{Key environment parameters}
\label{tab:parameter_settings}
\begin{tabular}{lrl}
\toprule
\textbf{Parameter} & \textbf{Value} & \textbf{Symbol} \\
\midrule
Carbon market price & \$60/tCO\textsubscript{2}e & $P$ \\
Result-based payment  & $[0, P]$ & R \\
Action payment max & \$100/ha & $a_{\max}$ \\
Initial SOC range & $\sim U[30, 80]$ tCO\textsubscript{2}e/ha & $C_j^0$ \\
Maximum SOC       & $200$ tCO\textsubscript{2}e/ha & $C_m$ \\
Carbon observation noise & 0.5 tCO\textsubscript{2}e & $\sigma$ \\
Crop selling price & \$0.27/kg & $w$ \\
MRV fixed cost & \$300 & $m_0$ \\
MRV variable cost & \$80 & $m_1$ \\
MRV scaling exponent & 0.7 & $\delta$ \\
Number of farmers  & 100 & $N$ \\
\bottomrule
\end{tabular}
\end{table}
The agronomic parameters in Table~\ref{tab:csp_practices} are based on the agricultural literature \citep{agroparam1,agroparam2, agroparam3, agroparam4} and have been validated by domain experts. They are intended as a plausible instantiation for one regional setting, i.e, Indian paddy farmers, rather than a definitive calibration.
SOC parameters are based on \citet{Lal_2015}.

\begin{table}[htbp]
    \centering
    \setlength{\tabcolsep}{1mm}
    \caption{Carbon sequestration practices with associated parameters.
AWD: alternate wetting and drying; INM: integrated nutrient management;
DRV: deep-root variety.}
    \label{tab:csp_practices}
    \renewcommand{\arraystretch}{1.25}
    \begin{tabular}{lccc}
        \toprule
        \textbf{Practice} & ${c}_i$ & ${k}_i$ & ${z}_i$ \\
         & \textbf{(tCO\textsubscript{2}e/ha)} & \textbf{(\$/ha)} & \textbf{(t/ha)} \\
        \midrule
        Crop Residue Retention                 & 1.10 & $-140$ & $+0.15$ \\
        Reduced/Minimum Tillage                  & 0.32 & $-80$  & $-0.10$ \\
        Mid-Season Drainage                      & 0.40 & $-80$  & $-0.10$ \\
        AWD             & 0.55 & $-80$  & $+0.05$ \\
        SRI Water Management                     & 0.60 & $0$    & $+0.30$ \\
        Optimised N Application                  & 0.35 & $+120$ & $+0.10$ \\
        INM           & 0.85 & $0$    & $+0.10$ \\
        Manure Application                       & 0.50 & $+125$ & $+0.12$ \\
        Biochar Application                      & 2.20 & $+400$ & $+0.30$ \\
        Raised Bed Cultivation                   & 1.50 & $+600$ & $+0.45$ \\
        Early-Maturing Variety                   & 0.30 & $-80$  & $0.00$  \\
        High-Biomass/DRV           & 0.70 & $-80$  & $+0.10$ \\
        \bottomrule
    \end{tabular}
\end{table}

Here we define 
\begin{equation}
y_i = 1000 \cdot z_i \cdot w
\end{equation}

\pagebreak

\section{Adoption Rate Change}
 Figure~\ref{fig:adoption_rate_het} illustrates how adoption rates evolve over the seasons, averaged across seeds and farmers. The gradual decline in adoption is driven by soil carbon saturation. As soil carbon accumulates over successive seasons, the marginal amount of additional carbon that can be sequestered decreases, reducing the expected carbon revenue from adopting the practices. Consequently, profit margins decline over time, making it increasingly difficult for the aggregator to learn contracts that are simultaneously attractive to farmers and profitable for the aggregator.

\begin{figure}[htbp]
    \centering
    \includegraphics[width=\columnwidth]{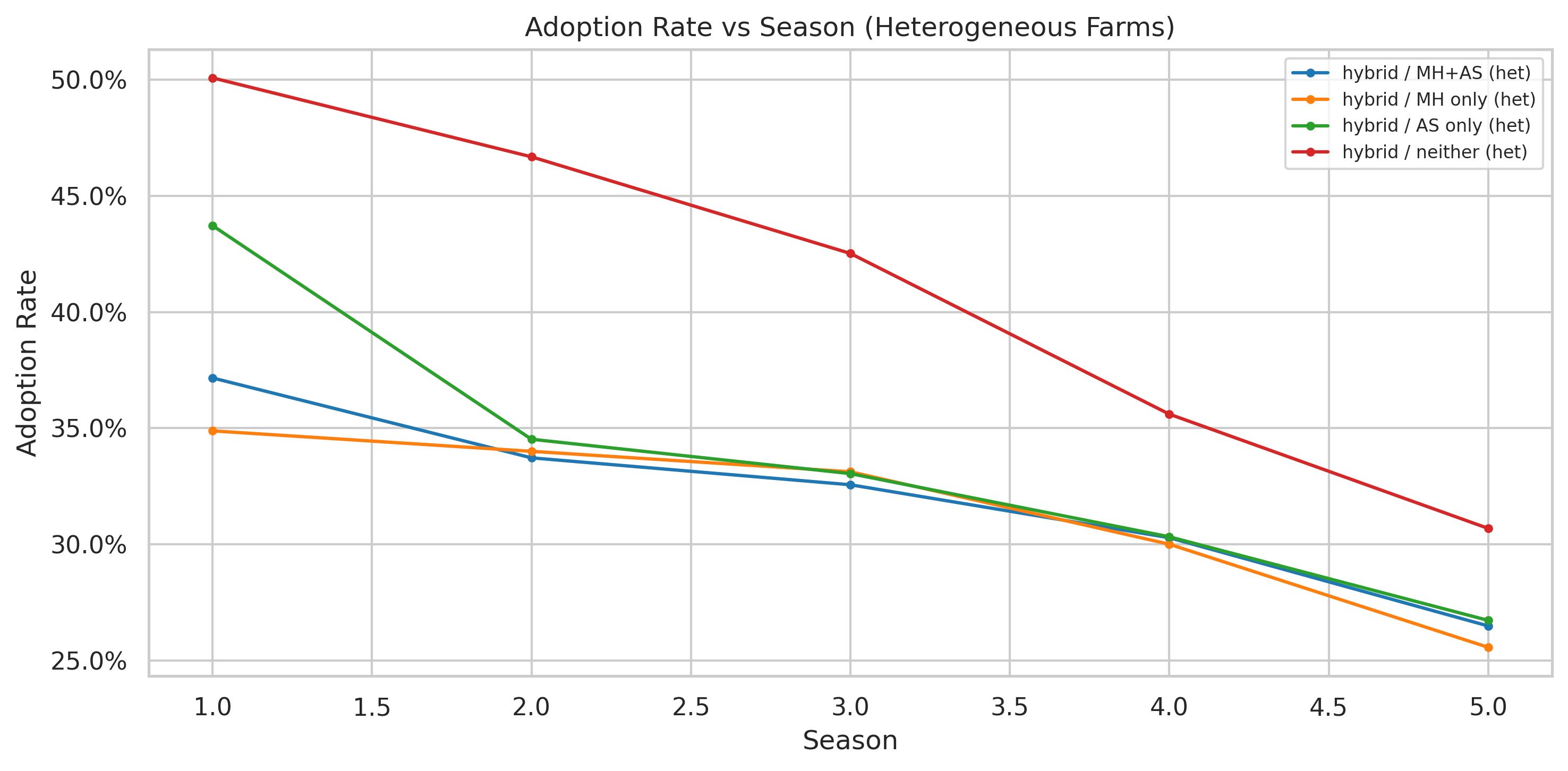}
    \caption{Adoption rates across seasons for the different conditions, variable farm size.}
    \label{fig:adoption_rate_het}
\end{figure}

\section{Sensitivity to the Cost-Limiter Effort Level $\underline{\mathbf{e}}$}
\label{app:e-sensitivity}

The cost-limiter interpretation of moral hazard depends on the effort level at
which the aggregator's per-hectare losses are evaluated. To check that our
conclusions are not an artifact of a particular choice, we re-ran the MH+AS and
AS-only conditions under heterogeneous farm size for
$e \in \{0.1, 0.2, 0.3\}$. Because this is a supplementary check, it uses a
reduced set of ten seeds; the MH+AS and AS-only values and p-value therefore differ slightly
from the 25-seed figures reported in the main text, and the intervals are
correspondingly wider.

Profit-capture is essentially flat in $e$ (Table~\ref{tab:e_sensitivity}), and all three rows of MH+AS exceed the AS-only
value of $0.546$, though the confidence intervals overlap at this sample size. All three paired Wilcoxon tests against AS are significant.

The price channel shows the effect most clearly. Paired by seed, the mean
action payment rate for adopted practices is lower under MH+AS than under
AS only at every effort level (Table~\ref{tab:ap_rate_sensitivity}). The direction is consistent across
the range and the magnitude varies. We do not attempt to attribute that
variation to a specific mechanism on ten seeds. The sweep establishes what
it was designed to, the cost-limiter effect. Moral hazard operating through
price rather than quantity under heterogeneous farm size is not an
artifact of the particular $e$ adopted in the main text.

\begin{table*}[htbp]
\centering
\begin{tabular}{lcccc}
\toprule
\textbf{Experiment} & \textbf{Mean} & \textbf{Std} & \textbf{95\% CI} & \textbf{$p$-value (vs AS only)} \\
\midrule
MH+AS (e=0.1) & \textbf{0.603} & 0.137 & [0.523, 0.691] & 0.037 \\
MH+AS (e=0.2) & \underline{0.603} & \underline{0.126} & [0.527, 0.682] & 0.002 \\
MH+AS (e=0.3) & 0.596 & 0.130 & [0.520, 0.680] &  0.002 \\
AS only (no MH) & 0.546 & \textbf{0.121} & [0.474, 0.622] & ---\\ 
\bottomrule
\end{tabular}
\caption{Profit-capture ratio across cost-limiter effort levels $e$.}
\label{tab:e_sensitivity}
\end{table*}

\begin{table*}[htbp]
\centering
\begin{tabular}{lccccc}
\toprule
\textbf{Experiment} & \textbf{AP Rate} & \textbf{AS only} & \textbf{Diff} & \textbf{95\% CI} & \textbf{$p$-value} \\
\midrule
MH+AS (e=0.1) & \$6.88 & \$8.92 & \$-2.04 & [-3.23, -0.86] & 0.0195 \\
MH+AS (e=0.2) & \$7.41 & \$8.92 & \$-1.51 & [-2.50, -0.64] & 0.0098 \\
MH+AS (e=0.3) & \$8.07 & \$8.92 & \$-0.85 & [-1.64, +0.12] & 0.1309 \\
\bottomrule
\end{tabular}
\caption{Mean action payment rate (\$/ha) for adopted practices: each cost-limiter effort level $e$ against \texttt{AS only}.}
\label{tab:ap_rate_sensitivity}
\end{table*}

\section{Cost Limiter Mechanism}
Consider this example for one isolated practice. Reduced tillage (base cost $-\$80$/ha, yield effect $-100$ kg/ha,
carbon $0.32$ tCO\textsubscript{2}e/ha) offered to a farmer with cost type
$\theta_j = 0.45$ and farm size $7$ ha, in the first season ($t=0$) with soil
at saturation deficit $1 - C_j/C_m = 0.80$, under a contract with
$R = \$10$/tCO\textsubscript{2}e and action payment $\$3$/ha:
\begin{align}
&\text{cost} = k^{\mathrm{eff}}_{ji} + \tfrac{1}{2}|k^{\mathrm{eff}}_{ji}| = -36 + 18 = -\$18/\text{ha} \\
\intertext{(realised cost via Eq.11 at $t=0$; the
type-scaled saving $k^{\mathrm{eff}}_{ji} = -80 \times 0.45 = -36$ is only half-realised early)}
&\text{yield\_revenue} = -100 \times \$0.27 = -\$27/\text{ha} \\
\intertext{(negative yield, felt immediately with no ramp, Eq.10)}
&\text{carbon\_payment} = 0.32 \times 0.80 \times \$10 = \$2.56/\text{ha} \\
\intertext{(practice contribution scaled by the saturation deficit)}
&\text{net} = -27 + 2.56 - (-18) = -\$6.44/\text{ha}
\end{align}
The net loss is therefore \$6.44/ha, making the practice independently unprofitable.
\begin{align}
\begin{split}
&\text{With MH } (e=0.2): \\
&\quad (0.2 \times (-6.44) + 3.00) \times 7 = \$11.98
\end{split} \\
\intertext{(scaled practice loss is only \$1.29/ha, which the \$3/ha action
payment more than covers, yielding $+\$1.71$/ha)}
\begin{split}
&\text{Without MH } (e=1.0): \\
&\quad (1.0 \times (-6.44) + 3.00) \times 7 = -\$24.08
\end{split}
\end{align}
Without MH, the full \$6.44/ha loss exceeds the \$3 AP, so the farmer does not
adopt.

\section{MRV Counterfactual}

In the counterfactual regime we remove the fixed per-farmer MRV component
($m_0 = 0$) and set the scaling exponent to $\delta = 1$, so that MRV cost is
strictly proportional to cultivated area and cost per hectare is constant
across farm sizes. To ensure that the counterfactual isolates the effect of
cost \emph{structure} rather than cost \emph{level}, we recalibrate the
variable rate $m_1$ so that total MRV liability over the farmer population is
unchanged from the baseline schedule.

Writing the baseline schedule as $m(f) = m_0 + m_1 f^{\delta}$ with
$m_0 = 300$, $m_1 = 80$, $\delta = 0.7$, and the counterfactual schedule as
$m'(f) = c\,f$, the neutrality condition over $N$ farms with areas $f_j$ is
\begin{equation}
\sum_{j=1}^{N} c\,f_j \;=\; \sum_{j=1}^{N}\left(m_0 + m_1 f_j^{\delta}\right),
\end{equation}
which solves to
\begin{equation}
c \;=\; \frac{m_0 N + m_1 \sum_j f_j^{\delta}}{\sum_j f_j}
\;=\; \frac{m_0 + m_1\,\mathbb{E}\!\left[f^{\delta}\right]}{\mathbb{E}[f]}.
\end{equation}
Expressing $c$ through the moments of the farm-size distribution rather than a
particular sample makes it a property of the population, so a single constant
can be held fixed across all seeds.

For $f \sim \mathrm{Uniform}[1,10]$ we have $\mathbb{E}[f] = 5.5$ and
\begin{equation}
\mathbb{E}\!\left[f^{0.7}\right]
= \frac{1}{9}\int_{1}^{10} f^{0.7}\,\mathrm{d}f
= \frac{10^{1.7} - 1}{1.7 \times 9}
\approx 3.210 .
\end{equation}
Substituting,
\begin{equation}
c = \frac{300 + 80(3.210)}{5.5} = \frac{556.8}{5.5} \approx 101.3 \text{ (rounding up)},
\end{equation}
that is, average baseline cost per farm divided by average area per farm. We
therefore set $m_1 = 101.3$ in the counterfactual.

Two features of the counterfactual deserve discussion. First, removing $m_0$ and setting $\delta = 1$ are not independent changes
but jointly define a single object. Both terms must be neutralised for it to be
size-invariant. Setting $m_0 = 0$ alone leaves a schedule still declining in $f$, and
the scale economy the counterfactual is meant to measure would survive in the control
condition.

Second, as seen in tables~\ref{tab:mrv_enrollment} and ~\ref{tab:mrv_enrollment_no_fixed}, realised MRV expenditure rises, but through enrolment rather than price. The results reported in these two tables are cumulative across the 5 seasons.
Total expected liability is equalised by construction, and cost per \emph{enrolled}
farmer falls even as the number of enrolled farmers grows: total expenditure is an
equilibrium outcome of the aggregator's contracting problem, not a primitive of the
schedule, and its increase is the mechanical consequence of covering farms the baseline
excluded. Reported cost per hectare is affected the same way: under $m'$ it equals $c$ whatever the
composition of the enrolled pool, whereas under $m$ it is an area-weighted average of a
per-hectare schedule that declines in $f$, so concentrating enrollment among larger farms
lowers it mechanically. The gap between the two regimes therefore reflects who enrols
rather than what MRV costs.

\begin{table*}[htbp]
\centering
\begin{tabular}{lcc}
\toprule
\textbf{Metric} & \textbf{First-Best} & \textbf{RL} \\
\midrule
Farmers enrolled & 295 [285, 305] & 160 [146, 175] \\
Land enrolled (ha) & 2010 [1943, 2084] & 1302 [1189, 1420] \\
Total MRV spend & \$177984 [\$172091, \$184220] & \$103426 [\$94201, \$112996] \\
MRV / farmer & \$604 [\$600, \$607] & \$646 [\$642, \$650] \\
MRV / ha & \$89 [\$88, \$89] & \$79 [\$79, \$80] \\
Agg. total profit & -- & \$84865 [\$78967, \$90747] \\
\bottomrule
\end{tabular}
\caption{Cumulative MRV expenditure and enrollment (MH+AS), $m_0=300, m_1=80, \delta=0.7$  (95\% CI).}
\label{tab:mrv_enrollment}
\end{table*}

\begin{table*}[htbp]
\centering
\begin{tabular}{lcc}
\toprule
\textbf{Metric} & \textbf{First-Best} & \textbf{RL} \\
\midrule
Farmers enrolled & 359 [350, 369] & 283 [242, 326] \\
Land enrolled (ha) & 1979 [1911, 2054] & 1568 [1337, 1803] \\
Total MRV spend & \$200498 [\$193326, \$208020] & \$158866 [\$135017, \$183366] \\
MRV / farmer & \$558 [\$546, \$570] & \$560 [\$550, \$570] \\
MRV / ha & \$101 [\$101, \$101] & \$101 [\$101, \$101] \\
Agg. total profit & -- & \$76292 [\$68227, \$84139] \\
\bottomrule
\end{tabular}
\caption{Cumulative MRV expenditure and enrollment (MH+AS), $m_0=0, m_1=101.3, \delta=1.0$ (95\% CI).}
\label{tab:mrv_enrollment_no_fixed}
\end{table*}

\section{Efficiency Breakdowns}

For completeness we report how welfare divides between the two parties under
each ablation, for the fixed farm-size setting (Table~\ref{tab:decomposition})
and the heterogeneous setting (Table~\ref{tab:decomposition_het}). Each entry
is expressed as a fraction of first-best welfare, so the three columns of a row
sum to one. Farmer surplus is the aggregate net payoff to adopting
farmers. Aggregator profit is as defined in the main text. Deadweight
loss (DWL) is the residual: the share of first-best welfare the learned contract
fails to realise, arising from factors such as the pooling constraint and the information frictions.

\section{Contract Payments}

Figure~\ref{fig:payment_structure_het} shows mean carbon and action payments as fractions of revenue. Action payments are favored over result based payments across all ablations. This may be due to allowing greater freedom to the aggregator to incentivise specific actions, however exact characterisation of this behaviour is left to future work.
\begin{table*}[htbp]
\centering
\begin{tabular}{lccc}
\toprule
\textbf{Ablation} & \textbf{Farmer Surplus} & \textbf{Aggregator Profit} & \textbf{DWL} \\
\midrule
MH+AS & 0.154 [0.138, 0.169] & 0.599 [0.537, 0.656] & 0.247 [0.178, 0.320] \\
MH only & 0.103 [0.091, 0.114] & 0.699 [0.639, 0.756] & 0.198 [0.135, 0.267] \\
AS only & 0.175 [0.161, 0.188] & 0.566 [0.521, 0.607] & 0.259 [0.208, 0.313] \\
Neither & 0.121 [0.108, 0.133] & 0.809 [0.761, 0.852] & 0.071 [0.020, 0.123] \\
\bottomrule
\end{tabular}
\caption{Profit-capture decomposition, fixed farm size (95\% CI).}
\label{tab:decomposition}
\end{table*}

\begin{table*}[htbp]
\centering
\begin{tabular}{lccc}
\toprule
\textbf{Ablation} & \textbf{Farmer Surplus} & \textbf{Aggregator Profit} & \textbf{DWL} \\
\midrule
MH+AS & 0.156 [0.142, 0.171] & 0.620 [0.577, 0.662] & 0.224 [0.170, 0.277] \\
MH only & 0.131 [0.119, 0.142] & 0.612 [0.574, 0.651] & 0.257 [0.212, 0.303] \\
AS only & 0.173 [0.160, 0.187] & 0.568 [0.529, 0.609] & 0.258 [0.207, 0.308] \\
Neither & 0.164 [0.151, 0.177] & 0.680 [0.641, 0.719] & 0.156 [0.105, 0.205] \\
\bottomrule
\end{tabular}
\caption{Profit-capture decomposition, variable farm size (95\% CI).}
\label{tab:decomposition_het}
\end{table*}

\begin{figure}[htbp]
    \centering
    \includegraphics[width=\columnwidth]{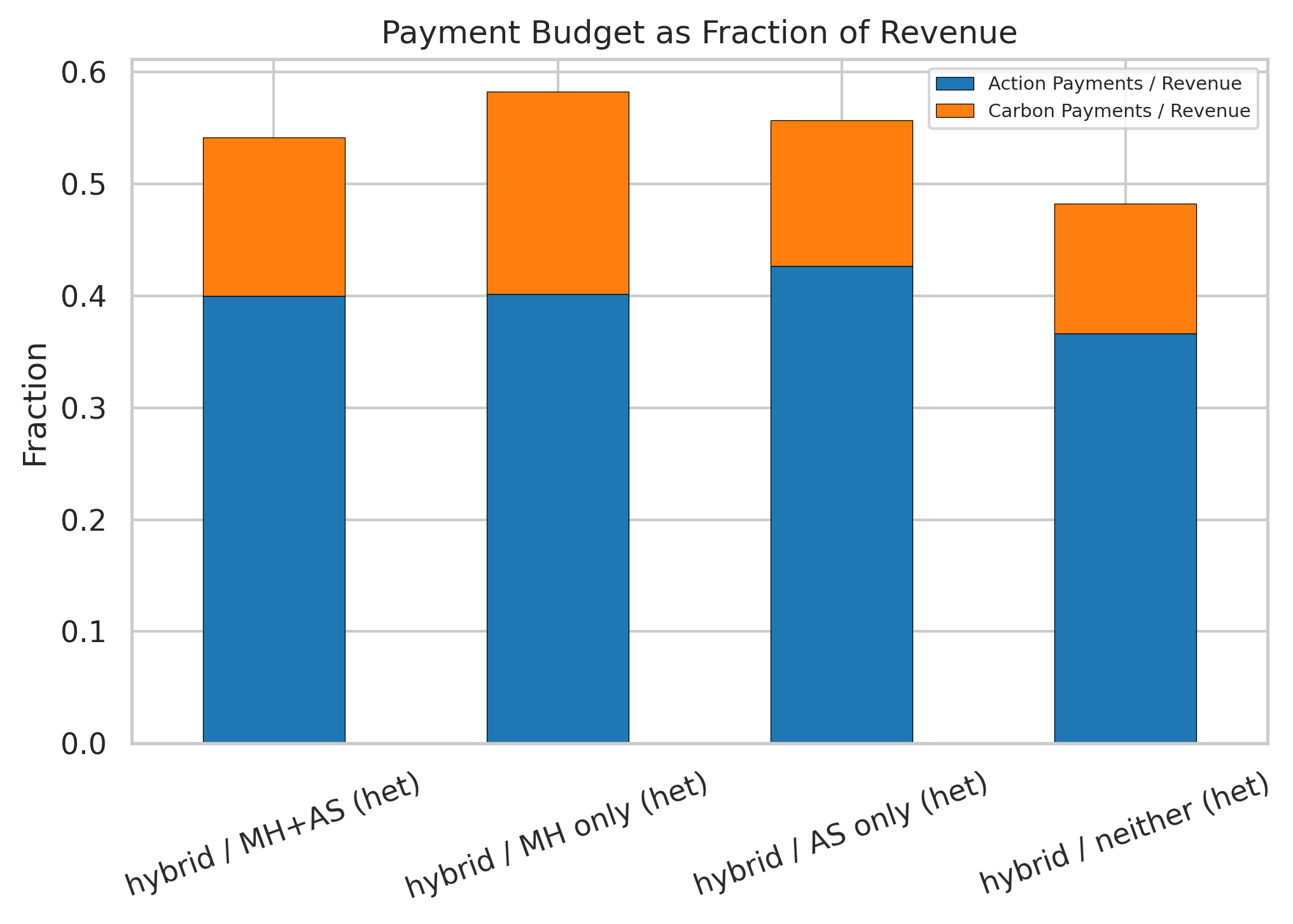}
    \caption{Contract payment structure, variable farm size.}
    \label{fig:payment_structure_het}
\end{figure}

\end{document}